\documentclass[journal]{IEEEtran}
\IEEEoverridecommandlockouts
\usepackage{cite}
\usepackage{amsmath,amssymb,amsfonts}
\usepackage{dsfont} 
\usepackage{algorithmic}
\usepackage{graphicx}
\usepackage{epstopdf} 
\usepackage{textcomp}
\usepackage{xcolor}
\usepackage{dblfloatfix} 
\usepackage[utf8]{inputenc}
\usepackage[english]{babel}
\usepackage{xspace} 
\usepackage{algorithm}
\usepackage{subcaption}
\usepackage{amsthm} 
\usepackage{mathtools} 
\usepackage[normalem]{ulem}
\usepackage{mathtools}

\usepackage{url}

\usepackage{balance}

\usepackage{multicol}
\usepackage{enumerate}
\usepackage{enumitem}
\usepackage{verbatim}

    \newtheorem{definition}{Definition}
    \newtheorem{theorem}{Theorem}
    \newtheorem{lemma}{Lemma}
    \newtheorem{corollary}{Corollary}
    \newtheorem*{newproblem}{The heterogeneous partition problem}
    \newtheorem*{UPSproblem}{The UAV Persistent Service problem}

\def\BibTeX{{\rm B\kern-.05em{\sc i\kern-.025em b}\kern-.08em
    T\kern-.1667em\lower.7ex\hbox{E}\kern-.125emX}}

\newcommand{\blue}[1]{{\color{black}{#1}}}

\catcode`\%=12
\newcommand\pcnt{
\catcode`\%=14

\begin{document}

\title{Optimizing UAV Recharge Scheduling \\ for Heterogeneous and Persistent Aerial Service$^{\star}$\thanks{${}^\star$A preliminary version of this manuscript appeared in the proceedings of IEEE GLOBECOM 2021~\cite{arribas2021optimal}.\vspace{-0mm}}}

\author{Edgar~{Arribas},
        Vicent~{Cholvi},
        and~Vincenzo~{Mancuso}
\thanks{E.~{Arribas} is with Universidad CEU San Pablo, Madrid (Spain).
E-mail: edgar.arribasgimeno@ceu.es.}
\thanks{V.~{Cholvi} is with Universitat Jaume I, Castell\'o (Spain). E-mail: vcholvi@uji.es}
\thanks{V.~{Mancuso} is with IMDEA Networks Institute, Legan\'es (Spain). E-mail: vincenzo.mancuso@imdea.org}
}

%

\markboth{Submitted to IEEE Transactions on Robotics}%
{Arribas \MakeLowercase{\textit{et al.}}: Bare Demo of IEEEtran.cls for IEEE Journals}

\maketitle

\begin{abstract}
The adoption of UAVs in communication networks is becoming reality thanks to the
deployment of advanced solutions for connecting UAVs and
using them as communication relays.
However, the use of UAVs
introduces novel energy constraints and
scheduling challenges in the dynamic management of network devices, due to the
need to call back and recharge, or substitute, UAVs that run out of energy.
In this paper, we design UAV recharging schemes under realistic assumptions on
limited flight times and time consuming charging operations.
Such schemes are designed to
minimize the size of the fleet to be devoted to a persistent service of a set of aerial locations, hence its cost.
We consider a fleet of homogeneous UAVs both under homogeneous and heterogeneous
service topologies.
For UAVs serving aerial locations with homogeneous distances to a recharge station,
we design a simple scheduling, that we name \textsc{HoRR}, which we prove to be feasible and optimal, in the sense that it uses the minimum possible
number of UAVs to guarantee the coverage of the aerial service locations.
For the case of non-evenly distributed aerial locations, we demonstrate that the problem becomes NP-hard,
and design a lightweight recharging scheduling scheme, \textsc{PHeRR}, that extends the operation
of \textsc{HoRR} to the heterogeneous case, leveraging the partitioning of the set of service locations.
We show that \textsc{PHeRR} is near-optimal because it approaches the performance limits identified
through a lower bound that we formulate on the total fleet size.
\end{abstract}

\begin{IEEEkeywords}
UAV, recharge schedule, optimization, persistent aerial service.
\end{IEEEkeywords}
\section{Introduction}

\IEEEPARstart{U}{nmanned}
aerial vehicles (UAVs), and lightweight drones in particular, are becoming attractive for service providers due to their ability to serve communication purposes and extend the capabilities of their fixed infrastructure. UAVs can be useful in many situations (e.g., in case of planned communication traffic surges due to massive meetings, disaster recovery missions, military applications, etc.), and they have played an important role during the COVID-19 pandemic to deliver goods and to irrorate disinfectants~\cite{9086010}. There is also a strong interest for UAVs in the IoT community, as they can be flexibly used for generating data and for harvesting data from fixed sensors. Therefore, many recent efforts tackle the integration of UAV-carried network nodes in cellular networks, either to control drone routes effectively or to experiment with relay schemes freshly introduced with 5G~\cite{8641421}.

With the current advances in communication technologies, the bottleneck in the adoption of UAVs no longer lays in architectural and protocol challenges and constraints, but rather in the limited energy that they can rely on. For this reason, flying several UAVs in a real scenario requires the accurate planning and monitoring of their energy consumption. With multiple UAVs and limited stations where the UAVs can land to get refueled, recharged or to get their batteries replaced, network designers need to solve new problems, and impose new constraints to their resource management algorithms. For instance, in a UAV-based tactical communication network or in a patrolling mission, UAVs have to be recharged cyclically while guaranteeing service at all times. The problem is threefold: $(i)$ the need to recharge a UAV (or change its battery) affects the service provided by the network of UAV; therefore, $(ii)$ the fleet of UAVs has to account for redundancy, so that when a UAV flies back to get fresh energy, the operation of the remaining drones remains consistent with the objectives of the mission; and $(iii)$ unlike traditional swapping schemes, the time during which a UAV with low energy goes offline to recharge is not negligible, since neither charging times nor the time to fly back and forth are negligible.

In this article, we formally define the problem of UAV recharge scheduling by accounting all associated overheads. Thereafter we study two scenarios: we start with a simple ideal case in which UAVs are dispatched at homogeneous distances from their recharging station, show that the problem admits optimal solutions, and we design \textsc{HoRR}, an optimal algorithm that implements an optimal solution in which UAV operational shifts repeat cyclically. Then we move to a more generic scenario in which UAV distances from the recharging station are heterogeneous. In this case the problem becomes NP-hard to solve, as we formally prove. We therefore resort to heuristics that generalize \textsc{HoRR}. Specifically we design the \textsc{HeRR} routine as the generalization of \textsc{HoRR}, obtained by accounting for some buffer time in the UAV shifts, which compensate for heterogeneous displacement distances that they have to cover in order to be recharged.  We improve the performance of \textsc{HeRR} by partitioning the fleet of UAVs into groups in which distances are as homogeneous as possible, and by applying \textsc{HeRR} to each group separately. The resulting algorithm, which we name \textsc{PHeRR} is shown to be near-optimal by comparing its performance with a lower bound of the problem.



This article significantly extends our preliminary results published in~\cite{arribas2021optimal}. That work focuses on the analysis of the ideal homogeneous scenario dealt with in this article and presents the \textsc{HoRR} algorithm and its optimality, which are also compactly illustrated in this manuscript in part of Section~\ref{s:scheduling}. However, most of the analysis and results derived in this article have not been previously published. In particular, the analysis of non-ideal heterogeneous cases and the derivation of the corresponding algorithms and properties is fully novel.

The rest of the article is organized as follows: Section~\ref{s:related} discusses the related work. Section~\ref{s:scenario} describes the reference UAV scenario studied in this article. Section~\ref{s:scheduling} presents the case of homogeneous distances to be covered by UAVs, and the optimality of  our solution. Section~\ref{s:hetero} illustrates the complexity of solving the recharge scheduling of UAVs under heterogeneous conditions. Section~\ref{sec:het_herr} proposes a near-optimal heuristic for the generic heterogeneous case, whose performance analysis is presented in Section~\ref{s:perf}. Finally, Section~\ref{s:conclusions} provides the conclusions.

\section{Related Work}
\label{s:related}
The increasing growth of the UAVs ecosystem during the last years has led to an increasing number of management strategies used to overcome both battery limitations as well as lack of available UAVs in a given situation.

Most of the work on energy management of UAV-based technologies focuses on the Vehicle Routing problem~\cite{MONTOYATORRES2015115}. Namely, the goal is to generate routes for a team of agents leaving a starting location, visiting a number of target locations, and returning back to the starting location. Among the many variants of such a problem, there is the possibility that the recharging stations in which the UAVs will be powered be either stationary or mobile~\cite{8660495}.
For instance, machine learning achieves near optimal results to solve UAV routing problems with recharging stops as studied in~\cite{ERMAGAN2022105524}, which shows results within a few percents from the optimal route.
Besides, the full-fledged automation of recharging stations has proven feasible in real testbeds~\cite{9138314}.
This opens to the advent of new automated applications and strategies, which involve pricing and recharge options. As an example, a credit-based game theory approach to UAV recharging at stationary stations is studied in~\cite{9019832}
and mechanisms to optimize the position of recharging stations have been proposed, e.g., in~\cite{COKYASAR2021146} and \cite{QIN2021107714} for fixed and mobile recharging stations, respectively.

The main aim of our work consists of the management of a fleet of UAVs intended to perform a persistent monitoring of a number of locations, which goes beyond energy management for routing problems. In~\cite{7846742}, the authors analyze the monitoring of a number of geographic areas, so that the objective was to minimize the number of UAVs that are needed to provide a continuous coverage. Each UAV was assumed to travel through the different areas, proposing a heuristic algorithm for such a task. However, the strategy of traveling through the different locations to be monitorized is not, in general, the best approach~\cite{hartuv2018scheduling}, and there is also the problem of finding the best cycle that UAVs must follow.

In~\cite{hartuv2018scheduling}, the authors show that the best replacement strategies are those in which each UAV to be replaced from a location directly returns to replace/recharge its battery (i.e., each UAV should monitor only one location). They also provide two approximation algorithms: one with an approximation factor upper bound of $1.5$ (when all the locations are known in advance) and the other with an average factor of $1.7$ (for the online version). They were followed by the authors of~\cite{8797808}, who consider minimizing the number of UAVs with multiple recharging stations. Using an approach similar to that in~\cite{hartuv2018scheduling}, in a subsequent work~\cite{9213932} the authors also considered the case with multiple recharging stations, showing that the problem is NP-hard, even for a single additional UAV (i.e., with just one back-up UAV needed to guarantee the service). They also provide two approximation algorithms for solving the problem, with approximation factors not worse that $1.6$ (offline) and $1.7$ (online), showing that they outperform the work of~\cite{8797808}.


\section{Reference Scenario}
\label{s:scenario}

We consider a set of UAVs that must perform a persistent task in a set $\mathcal{N}$ of aerial locations. We say that a UAV is \emph{covering/providing service} when it is at an aerial target location to perform the persistent task. Clearly, as time passes by, UAVs consume energy, and therefore they will periodically need to go to a \emph{recharging station}~(\emph{RS}) to recharge.

\begin{figure}
    \centering
    \vspace{1mm}
    \includegraphics[trim={5mm 3mm 4mm 3mm}, width=0.8\linewidth]{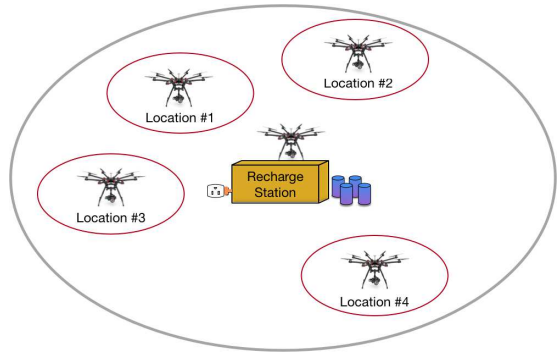}
    \vspace{-0mm}
    \caption{Scenario of the UAV persistent  problem.}
    \vspace{-2mm}
    \label{fig:ilustration}
\end{figure}

\paragraph{The RS} We consider that there is a single \emph{RS}, which provides a lot of flexibility, as it can be easily located in the most suitable place. 
When a UAV lands on the \emph{RS}, an automatic device will take care of replacing its battery, so that the UAV can be fully operational in short time. Alternatively, it is also possible that a UAV recharges its current battery at the \emph{RS}, although that will take significantly longer. Figure~\ref{fig:ilustration} illustrates the above mentioned scenario.

\paragraph{UAVs} In our work, we assume a fleet composed by identical UAVs. We denote as $f$ the maximum \emph{flight time} of each UAV (i.e., the battery life) and as $c$ the time it takes to replace/recharge its battery (which also depends on the \emph{RS}).

At this point, we remark that, although our results are valid regardless of the value of $c$, in our illustrative examples, as well as in the numerical analysis, we will assume that $c=15$~s. We used  this value  since, on one hand, the current battery exchange technology is mature enough to make such replacements safely. And, on the other hand, different studies have found  that such replacements can be carried out (from landing to take off) in less than $15$~s~\cite{michini2011automated,liu2017quado}.

\paragraph{Locations and displacements} UAVs need to cover aerial locations situated at arbitrary distances from the \emph{RS}. We denote as $g_i$ the \textsl{displacement time} that UAV needs to fly from the \emph{RS} to location $i\!\in\!\mathcal{N}$ (or viceversa). Hence, the time that a UAV requires to cover the location $i$ is $c+2g_i$ (i.e., the battery's replace/recharging time plus twice the displacement time). Since to cover a location $i$ a UAV needs to be able to, at least, fly to it and come back (which takes $2g_i$ time units) before its battery runs out (i.e., before $f$ time units), we assume that $2g_i\!<\!f$, for all $i\!\in\!\mathcal{N}$.

\,
\,
\,
\,
\begin{UPSproblem}
With the scenario described above, the UAV Persistent Service (UPS) problem consists in $(i)$ finding a recharge scheduling so that, at each time instant, each location in $\mathcal{N}$ is covered  by one UAV, and $(ii)$ doing it with the minimum number of UAVs. The recharge scheduling must instruct each UAV when to fly and cover a given location, and when to go to the \emph{RS} and replace/recharge its battery.
\end{UPSproblem}



\section{Homogeneous Scenarios}
\label{s:scheduling}

In this section, we consider the case in which the distance from the \emph{RS} to each of the aerial locations is homogeneous (i.e.,  $g_i\equiv g$, $\forall i\!\in\!\mathcal{N}$).

First, we define the \emph{Homogeneous Rotating Recharge} (\textsc{HoRR}) algorithm, and show that it solves the \emph{UPS} problem. Then, we provide some results regarding how the UAVs are instructed to recharge and prove that  \textsc{HoRR} is optimal, in the sense that it minimizes the number of  UAVs.

\subsection{The \textsc{HoRR} algorithm}

The rationale behind how the algorithm has been designed is based on the fact that the distances to the locations to be covered are homogeneous. Thus, the UAVs that cover the locations are cyclically replaced at fixed time intervals, ensuring that they will provide service in the locations for as long as possible, and always replacing the UAV with the lowest energy.

The code of the \textsc{HoRR} algorithm is shown in Algorithm~\ref{a:scheduling}. It works as follows: at each time interval of $x$ time units (Steps~\ref{st:x-g} and~\ref{st:x}), the UAV with less energy goes to recharge, regardless of whether or not it is actually running out of energy. In addition, $g$ time units before that UAV is instructed to recharge, a backup UAV is sent to replace it, so that the coverage  is maintained at all times. On its side, a recharged UAV is considered as a backup UAV.

\begin{algorithm}[t]
  \caption{\small\emph{Homogeneous Rotating Recharge (\textsc{HoRR})}\normalsize}
  \label{a:scheduling}
  \begin{algorithmic}[1]
    \REQUIRE{$\cal N$, $f$, $g$, $c$.}
    \STATE Obtain \small$x=\frac{f-2g}{N}$\normalsize, where $N=\mid\mathcal{N}\mid$.
    \STATE Initially, one UAV is instructed to provide service at each of the $\cal N$ aerial  locations.
    \STATE After $x\!-\!g$ time units, a fully charged backup UAV $u_c$ takes off from the \emph{RS} and goes to the location of the UAV with less energy $u_e$ (breaking ties arbitrarily). \label{st:x-g}
    \STATE When $u_c$ arrives at the location of $u_e$, it replaces $u_e$ and $u_e$ goes to recharge.  Once $u_e$ is recharged it will be considered as a backup UAV.
    \label{st:x}
    \STATE Go back to Step~\ref{st:x-g}.
  \end{algorithmic}
\end{algorithm}
\setlength{\textfloatsep}{15.5pt}

\begin{figure*}
    \centering
    \vspace{1mm}
    \includegraphics[trim={5mm 3mm 4mm 3mm}, width=0.65\linewidth]{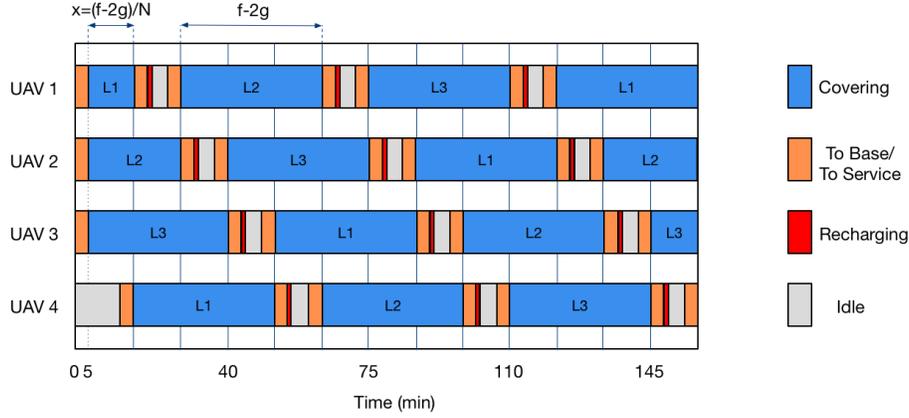}
    \vspace{-0mm}
    \caption{UAV recharge scheduling by using \textsc{HoRR} with $N=3$, $f=45$~min, $g=5$~min, $c=15$~s.}
    \vspace{-2mm}
    \label{fig:ODRSexample}
\end{figure*}

Note that \textsc{HoRR} assumes that there will always be a backup UAV ready to replace any other UAV instructed to recharge. In the following theorem we prove that, by using \textsc{HoRR}, the number of backup UAVs that guarantees that each location in $\cal N$ is permanently covered is \small$\left\lceil\frac{c+2g}{f-2g}N\right\rceil$\normalsize\!.

\begin{theorem}
Assume a fleet of UAVs that provide service in a homogeneous scenario so that the resulting system is characterized by $f$, $c$ and $g$. \textsc{HoRR} guarantees that $N$ locations can be permanently covered by using \small$M=N+  \left\lceil \frac{c+2g}{f-2g}N \right\rceil$\normalsize ~UAVs.
\label{the:HoRR}
\end{theorem}
\begin{proof}
According to Algorithm~\ref{a:scheduling}, at each time instant $kx$ (with $k\!\in\!\mathds{N}$) a UAV $u_e$ is instructed to recharge, so that at time instant $kx\!-\!g$ a backup UAV $u_c$ takes off to replace $u_e$ just on time. After that, it will take at least $c+2g$ time units for $u_e$ to be back and replace another UAV called back for recharging. During that interval, exactly $n=\left\lfloor\frac{c+2g}{x}\right\rfloor$ UAVs will be instructed to recharge, at intervals of $x=\frac{f-2g}{N}$ units after $kx$.
If the ratio $\frac{c+2g}{x}$ is integer, $u_e$ will be used for the $n$-th replacement, otherwise it will be used for the $(n+1)$-th replacement. In both cases, the number of UAVs instructed to recharge and not yet back to service is exactly $\left\lceil\frac{c+2g}{x}\right\rceil = \left\lceil\frac{c+2g}{f-2g}N\right\rceil$. Hence, this is also the number of backup UAVs needed by \textsc{HoRR}, and the proof follows.
\end{proof}

In Figure~\ref{fig:ODRSexample}, we show an illustrative example of how the \textsc{HoRR} algorithm works. We consider a scenario formed by three locations (i.e., $N\!=\!3$) and with $f=45$~min, $g=5$~min and $c\!=\!15$~s. Under those premises, Theorem~\ref{the:HoRR} guarantees that only one additional UAV is strictly necessary to guarantee a persistent coverage at the three locations (i.e., $M\!=\!4$).
Thus, every $x\!=\!11.\overline{6}$~min, one active UAV is instructed to recharge, and $5$~min in advance a fully charged backup UAV is also instructed to fly and replace that UAV. Observe also that, at the regime level of the scheduling (i.e., after the second recharge since the initial deployment of UAVs with full batteries)  each UAV provides service for $f\!-\!2g\!=\!35$~min.

\subsection{Recharging in the \textsc{HoRR} algorithm}

Next, we provide two results regarding when UAVs are instructed to recharge.

\begin{lemma}
\label{l:instruction}
By using \textsc{HoRR}, a UAV covering a location $i\!\in\!\mathcal{N}$ is instructed to recharge for the $k$-th time at time instant:

  \vspace{-3mm}
  \small
  \begin{eqnarray}
    t_i^k = \left( i + (k\!-\!1)N + (k\!-\!1) \left\lceil\frac{2g\!+\!c}{x}\right\rceil \right) x, \nonumber
  \end{eqnarray}
  \vspace{-3mm}
  \normalsize

  \noindent
  where \small$x=\frac{f-2g}{N}$\normalsize.
\end{lemma}
\begin{proof}
  We prove the lemma by induction. Take $k=1$. Without loss of generality, assume that $u$ is the UAV instructed to cover the $i$-th location in the 1st round (otherwise, UAVs can be resorted). Then:

  \small
  \vspace{-4mm}
  \begin{eqnarray}
    t_i^1 = ix, \nonumber
  \end{eqnarray}
  \normalsize
  \vspace{-6mm}

  \noindent
  which satisfies the lemma.

  Assume the lemma is true for a given $k$. We prove that then the lemma is also true for $k\!+\!1$.

  According to the inductive hypothesis, $u$ is instructed to recharge for the $k$-th time at:

  \small
  \vspace{-3mm}
  \begin{eqnarray}
    t_i^k = \left( i + (k\!-\!1)N + (k\!-\!1) \left\lceil\frac{2g\!+\!c}{x}\right\rceil \right) x. \nonumber
  \end{eqnarray}
  \normalsize
  \vspace{-2mm}

  Then, $u$ arrives to the \emph{RS}  at time \small$t_i^k\!+\!g$\normalsize~and takes off at \small$t_i^k\!+\!g\!+\!c$\normalsize~(i.e., after it is fully recharged). This means that $u$ can replace another UAV at time \small$t_i^k\!+\!2g\!+\!c$\normalsize~or later.

  Following the scheduling, $u$ will replace another UAV at instant $jx$, for some $j\!\in\!\mathds{N}$. Concretely, it will do it at the minimum time instant $jx$ such that \small$jx\!\geq\! t_i^k\!+\!2g\!+\!c$\normalsize. Hence:

  \small
  \vspace{-3mm}
  \begin{align}
    j & \!=\! \left\lceil\frac{t_i^k+2g+c}{x}\right\rceil
    \!=\! \left\lceil i + (k\!-\!1)N \!+\! (k\!-\!1)\!\left\lceil\frac{2g+c}{x}\right\rceil\!+\!\frac{2g+c}{x}\right\rceil \nonumber\\
    & = i+(k\!-\!1)N + (k\!-\!1)\left\lceil\frac{2g+c}{x}\right\rceil+\left\lceil\frac{2g+c}{x}\right\rceil \nonumber\\
    & = i + (k\!-\!1)N + k\left\lceil\frac{2g+c}{x}\right\rceil\!\!. \nonumber
  \end{align}
  \normalsize
  \vspace{-2mm}

Then, after $Nx$ time units, $u$ will be instructed to recharge again for the ($k\!+\!1$)-th time at time instant:

  \small
  \vspace{-4mm}
  \begin{align}
    & \!\!\!t_i^{k+1} = jx+Nx
    = \left(i+(k-1)N+k\left\lceil\frac{2g+c}{x}\right\rceil\right)x+Nx \nonumber\\
    &
    = \left(i + kN + k \left\lceil\frac{2g + c}{x}\right\rceil\right) x, \nonumber
  \end{align}
  \normalsize
  \vspace{-3mm}

 which proves the lemma.
\end{proof}

\begin{corollary}
By using \textsc{HoRR}, any UAV is instructed to recharge every \small$\left(N \!+\! \left\lceil\frac{2g+c}{x}\right\rceil\right)\!x$ \normalsize time units.
  \label{cor:optimal}
\end{corollary}
\begin{proof}
The difference between two consecutive recharges at the same location $i$ is given by

\small
  \vspace{-4mm}
  \begin{align}
      \! t_i^{k+1} - t_i^k  & = \left(i + kN + k\left\lceil\frac{2g+c}{x}\right\rceil\right)x  \,-\, \bigg(i + (k\!-\!1)N  \nonumber\\
    & +  \left . (k\!-\!1)\left\lceil\frac{2g+c}{x}\right\rceil\right)x = \left(N + \left\lceil\frac{2g+c}{x}\right\rceil\right)x, \nonumber
  \end{align}
  \vspace{-3mm}
  \normalsize

  and hence the corollary follows.
\end{proof}

\subsection{Optimality of the \textsc{HoRR} algorithm}
\label{ss:optimal}


In the following theorem, we show which is the strictly minimum number of UAVs necessary to guarantee that a given set of locations are covered in a persistent manner.

\begin{theorem}
Assume a fleet of UAVs that provide service in a homogeneous scenario so that the resulting system is characterized by $f$, $c$ and $g$. The minimum number of UAVs necessary to guarantee that $N$ of them will be always providing service is \small $M=N+ \left\lceil \frac{c+2g}{f-2g}N \right\rceil$\normalsize.
\label{th:optimal}
\end{theorem}
\begin{proof}
Consider a UAV providing service at a given aerial location $i\!\in\!\mathcal{N}$.
Such a UAV can provide service to $i$ for $\phi-2g$ time units, with $\phi \in [2g, f]$, and follow a duty cycle lasting $\phi + c + \alpha$, with $\alpha \ge 0$ indicating for how long the UAV remains idle after recharging. The fraction of time dedicated to serve location $i$ is therefore  \small$r(\phi, \alpha)=\frac{\phi-2g}{\phi+c+\alpha}$\normalsize, which is maximized for $(\phi,\alpha) = (f,0)$. Therefore, the minimum number of UAVs needed to serve $i$ is $n$ such that $n \cdot r(f,0)= 1$, so that the location be covered $100\%$ of the time. This means that we need, at least, \small$n = \frac{f+c}{f-2g}$\normalsize~UAVs to cover one location. Hence, a lower bound for the number of backup UAVs dedicated to one location is $n -1= \frac{f+c}{f-2g}-1 = \frac{c+2g}{f-2g}$ (we subtract $1$, since the covering UAV is not counted as backup UAV).
With the above, we can obtain a non-integer lower bound, because we are considering the average behavior of UAVs, which could be used to provide service to a different location after each recharge using an optimal scheduler.
To cover $N$ locations, being the scenario homogeneous, we need at least $(n-1)\,N$ backup UAVs in total, and round this number to the next closer integer, i.e., $\left\lceil \frac{c+2g}{f-2g}N \right\rceil$\normalsize, which is therefore the lower bound under any scheduling scheme and homogeneous assumptions.

\end{proof}


The proof of the above theorem implies that \textsc{HoRR} is optimal because, according to Theorem~\ref{the:HoRR}, it uses the minimum possible number of backup UAVs.

\begin{corollary}
\label{cor:optimal}
\textsc{HoRR} is optimal.
\end{corollary}



\subsection{Numerical analysis of the \textsc{HoRR} algorithm}
\label{ss:HoRR_analysis}

To end this section and through a numerical analysis of the result provided by Theorem~\ref{the:HoRR}, here we illustrate how the fleet size $M$ grows as a function of the number of locations to be covered, $N$. Figure~\ref{fig:HoRR_beh} shows that relationship for different values of $f$. We have also considered different values of both $g$ and $c$, and we observed that the shapes were similar.

The first observation is that the values of $M$ grow linearly with the values of $N$. This can be readily explained as follows: We know that $M=N+  \left\lceil \frac{c+2g}{f-2g}N \right\rceil$, which can be rewritten as $M/N = 1 + \left\lceil\frac{c+2g}{f-2g}N\right\rceil/N \simeq \frac{f+c}{f-2g}$. However, in any concrete scenario  the parameters $f$, $g$ and $c$ remain constant, since they model  features that do not change. Therefore, we have that $M$ linearly grows with $N$ at a rate of $\frac{f+c}{f-2g}$. At this point, we note that the steps that can be observed in the graph are due to rounding up the number of UAVs.

Another observation is that the higher the value of $f$, the lower the slope. Again, this can be explained since the value of $\frac{f+c}{f-2g}$ decreases with the increase in the value of $f$, until it reaches $1$ (which happens when $f$ is much larger than both $g$ and $c$.

\begin{figure}
\centering
\includegraphics[width=8cm]{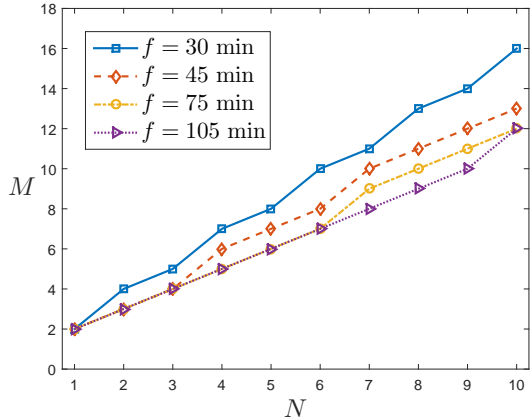}
\caption{Behavior of the {\sc HoRR} algorithm. $g=5$~min and $c=15$~s.}
\label{fig:HoRR_beh}
\end{figure}
%

\section{Heterogeneous Scenarios: NP-hardness and covering cost}
\label{s:hetero}

In this section, we address the case in which the distance between the \emph{RS} and each different location can be different. First, we show that, in that scenario,  the \emph{UPS} problem is NP-hard. Then, we prove that covering heterogeneous scenarios is, in general, more costly than covering homogeneous ones.

\subsection{NP-hardness}


\begin{theorem}
\label{the:het}
  The UPS problem in the heterogeneous case is \!NP-hard.
\end{theorem}
\begin{proof}

    Consider an instance of the \emph{UPS} problem such that $c\!=\!0$ and $2g_i\!<\!f/2$, $\forall 1\!\leq \!i\!\leq \!N$. This instance of the \emph{UPS} problem is equivalent to the general instance of the \emph{Minimal Spare Drone for Persistent Monitoring (MSDPM)} problem~\cite{8901097} with one \emph{RS}. At this point, we note that the \emph{MSDPM} problem is equivalent to the \emph{ Bin Maximum Item Double Packing (BMIDP)} problem (see~\cite[Lemma~5.1]{8901097}). In addition, the \emph{BMIDP} problem is NP-hard, as we formally prove in Appendix~\ref{App:NP-completeness}.
    Therefore, we have that the \emph{MSDPM} problem is NP-hard and, consequently, the \emph{UPS} problem is also NP-hard.
\end{proof}

This result shows that, contrary to what happens in the homogeneous case, in the heterogeneous one it is not possible to find an optimal scheduling that works in polynomial-time.


\subsection{Covering Cost}

In this subsection, we compare the covering cost, in terms of the number of UAVs, required by homogeneous scenarios against the cost required by heterogeneous ones. To do this, we first obtain a lower bound on the necessary number of UAVs  to guarantee that $N$ locations will be permanently covered.

\begin{theorem}
Assume a fleet of UAVs that provide service in a heterogenous scenario so that the resulting system is characterized by $f$, $c$ and $g_i$ (for each $i\!\in\!\mathcal{N}$). A lower bound on the minimum  number of UAVs necessary to guarantee that $N$ of them will be always providing service is
  $M_{LB} = N + \left\lceil\sum\limits_{i=1}^N \frac{c+2g_i}{f-2g_i}\right\rceil$.
  \nonumber
\label{th:LBvin}
\end{theorem}

\begin{proof}
Similarly to what argued in the proof of Theorem~\ref{th:optimal}, given an aerial location $i\!\in\!\mathcal{N}$, a UAV provides service to $i$ for a fraction of its
duty-cycle time, and such fraction can reach the maximum value of $\frac{f-2g_i}{f+c}$.
This means that we need, at least, \small$\frac{f+c}{f-2g_i}$\normalsize~UAVs to cover that location $i$. Hence, a lower bound for the number of backup UAVs for location $i$ is $\frac{f+c}{f-2g_i}-1 = \frac{c+2g_i}{f-2g_i}$, which gives an average value.

Summing over all possible aerial locations, and taking the ceiling, we get a lower bound of the total number of backup UAVs, and the theorem follows.
%
\end{proof}


%

Now, we can use the obtained lower bound to show that covering in heterogeneous scenarios is, in general, more costly than covering in homogeneous ones.

\begin{theorem}
Assume a fleet of UAVs that provide service in a heterogenous scenario so that the resulting system is characterized by $f$, $c$ and $g_i$ (for each $i\!\in\!\mathcal{N}$). Let $M_{het}$ be the minimum number of UAVs that guarantee that $N$ of them are always providing service at the target locations, and let $M_{hom}$ be the minimum number of UAVs that guarantee that $N$ of them are always providing service in the homogeneous scenario when $g=Avg(g_i)$.  Then,
 $M_{het} \geq M_{hom}$.
 \label{the:LBhomo}
\end{theorem}
\begin{proof}
  From Theorem~\ref{th:LBvin}, we know that \small$M_{het}\!\geq\! N \!+\! \left\lceil\sum\limits_{i=1}^N\frac{c+2g_i}{f-2g_i}\right\rceil$\normalsize\!. Now, we apply Theorem~\ref{th:sumNsumsum} from Appendix~\ref{app:MathInequality} to \small$\sum\limits_{i=1}^N\frac{c+2g_i}{f-2g_i}$\normalsize:
  \begin{eqnarray}
    \sum\limits_{i=1}^N \!\frac{c\!+\!2g_i}{f\!-\!2g_i} \geq N \frac{\sum\limits_{i=1}^N \!c\!+\!2g_i}{\sum\limits_{i=1}^N \!f \!-\! 2g_i} = N\frac{Nc\!+\!2Ng}{Nf\!-\!2Ng} = N\frac{c\!+\!2g}{f\!-\!2g}. \nonumber
  \end{eqnarray}
  Hence, $M_{het}\!\geq\! N \!+\! \left\lceil\frac{c+2g}{f-2g}N\right\rceil$. Since the optimal number of UAVs for the homogeneous scenario is $M_{hom}\!=\!N \!+\! \left\lceil\frac{c+2g}{f-2g}N\right\rceil$ (see Theorem~\ref{th:optimal}), then $M_{het}\geq M_{hom}$.
\end{proof}

\section{Heterogeneous Scenarios: the {\sc PHeRR} Algorithm}
\label{sec:het_herr}

In this section, we introduce a UAV recharge scheduling algorithm for heterogeneous scenarios. Such an algorithm works in two phases: in the first phase the whole set of locations are properly partitioned into subsets so that, in the second phase, a recharging scheduling routine is individually applied to each  of the resulting subsets.


The rationale behind partitioning the whole set of locations is to work with more homogeneous subsets. As it will be clear later, this will prevent the furthest locations, which can only be covered for less time, from affecting the coverage of closest locations.

 \subsection{The {\sc HeRR} routine}

\begin{algorithm}[t]
  \caption{\small \emph{Heterogeneous Rotating Recharge} ({\sc HeRR})\normalsize}
  \label{a:HeRR}
  \begin{algorithmic}[1]
    \REQUIRE{$\mathcal{I} \subseteq \mathcal{N}$, $f$, $\{g_{i_j}\}_{i_j\in\mathcal{I}}$, $c$.}
    \STATE Sort $\{g_{i_j}\}$ in increasing order.
    \STATE For all $1\!\leq j\!\leq \!I$: obtain \small$x_{i_j}=\frac{f-2g_{i_I}}{\sum_{l=1}^I g_{i_l}}\cdot g_{i_j}$\normalsize, where $I= \mid \!\mathcal{I}\! \mid$.
    \STATE Initially, one UAV is instructed to provide service at each of the ${\cal I}$ aerial locations.
    \STATE Set $i_j\leftarrow 1$.
    \STATE After $x_{i_j}\!-\!g_{i_j}$ time units, a fully charged backup UAV $u_c$ takes off from the \emph{RS} and goes to  location ${i_j}$. \label{st:xi-gi}
    \STATE When $u_c$ arrives to location ${i_j}$, it replaces the UAV that is covering it, which goes to recharge. Once recharged, that UAV will be considered as a backup UAV.\label{st:xi}
    \STATE Set $j\leftarrow \left(j \!\!\!\mod\! I\right)  + 1$ and go back to Step~\ref{st:xi-gi}.
  \end{algorithmic}
\end{algorithm}
\setlength{\textfloatsep}{15.5pt}

First, we introduce the recharge scheduling routine that is used once the whole set of locations has been partitioned. That routine, which we call \emph{Heterogeneous Rotating Recharge} ({\sc HeRR}), is a generalization of the {\sc HoRR} algorithm that  takes into account that  distances from the \emph{RS}  to the locations could be different.

Its code is shown in Algorithm~\ref{a:HeRR}, and it works as follows: Let $\mathcal{I}$ be a subset of $I$ locations obtained after partitioning $\mathcal{N}$. For each location $i_j \in \mathcal{I}$, every $x_{i_j}$ time units (Steps~\ref{st:xi-gi} and~\ref{st:xi}) the UAV that covers $i_j$ will go to recharge, regardless of whether or not it is actually running out of power. In addition, $g_{i_j}$ time units before that UAV is instructed to go to recharge, a backup UAV is sent to replace it, so that the coverage is maintained at all times. On its side, a recharged UAV is considered as a backup UAV.
As it can be seen, the main difference between {\sc HoRR} and {\sc HeRR} is that now the time instants at which UAVs are instructed to recharge are not equally spaced, but have been chosen so that no UAV will run out of energy before reaching the \emph{RS}.



For simplicity, from now on we consider that if $l\!> I$ then $g_{i_l}\!=\!g_{i_j}$, where $1\!\leq\!j\!\leq \!I$ is the only number such that $l\equiv j \!\!\mod I$ (the same applies for $x_{i_l}$). 

In the following theorem, we provide a bound on the number of UAVs that guarantees that, by using the \textsc{HeRR} routine, each location in $\mathcal{I}$ is permanently covered.

\begin{theorem}
\label{the:HeRR}
Assume a fleet of UAVs that, by using {\sc HeRR}, provide service in a heterogenous scenario, and the resulting system is characterized by $f$, $c$ and $g_{i_j}$ (for each ${i_j}\!\in\!\mathcal{I}$). A sufficient number of UAVs necessary to guarantee that $I$ of them will be always providing service is:
\begin{eqnarray}
  M = I + \max_{1\leq k\leq I} \{n_k\}, \nonumber
  \label{eq:suff}
\end{eqnarray}
where $n_k \!=\! \min\limits_{n\in\mathds{N}}  \left\{\!n:\! \sum\limits_{l=k+1}^{k+n} \!g_{i_l} \!\geq\! \frac{g_{i_k}+c+g_{i_{k^*}}}{f-2g_{i_I}}\sum\limits_{j=1}^I \!g_{i_l}\!\right\}$, and $i_{k^*} \!=\! \min\left\{i_\alpha>i_k: \!\sum\limits_{l=k+1}^\alpha \!x_{i_l} \geq g_{i_k}\!+\!c\!+\!g_{i_\alpha}\right\}$. \\

\end{theorem}
\begin{proof}
 According to Algorithm~\ref{a:HeRR}, at some time instant $L\left(f-2g_{i_I}\right) + \sum_{l=1}^k x_{i_l}$ for some $L\!\in\!\mathds{Z}_{\geq0}$, $1\!\leq k\!\leq \!I$,~a UAV $u_e$ that is covering location $i_j\!=\!i_k$ is instructed to recharge. UAV $u_e$ goes to the \emph{RS} while a backup UAV $u_c$ takes off at $L\left(f-2g_{i_I}\right)+\sum_{l=1}^k x_{i_l} \!-\! g_{i_k}$ to replace $u_e$ at the proper instant. While $u_e$ gets ready, other $n_k$ UAVs are instructed to recharge.
 Hence, the first location that $u_e$ will be able to be ready to replace the next time is \small$i_{k^*} \!=\! \min\left\{i_\alpha>i_k: \!\sum\limits_{l=k+1}^\alpha \!x_{i_l} \geq g_{i_k}\!+\!c\!+\!g_{i_\alpha}\right\}$\normalsize.
 Thus, the time needed by  $u_e$ to be able to replace another location is $g_{i_k}\! + \!c\! + \!g_{i_{k^*}}$. Hence, it is sufficient to have  $n_k$ backup UAVs ready to replace the $n_k$ UAVs that are being instructed to recharge during this period such that \small$\sum\limits_{l=k+1}^{k+n_k} \!x_{i_l} \!\geq\! g_{i_k}\!+\!c\!+\!g_{i_{k^*}}$\normalsize.
 According to the definition of each $x_{i_l}$, the minimum $n_k$ that accomplishes this~is: \begin{eqnarray}
   n_k = \min_{n\in\mathds{N}} \left\{n: \sum\limits_{l=k+1}^{k+n} g_{i_l} \geq \frac{g_{i_k}+c+g_{i_{k^*}}}{f-2g_{i_I}} \sum\limits_{l=1}^I g_{i_l}  \right\}. \nonumber
   \label{eq:nk}
 \end{eqnarray}

\begin{figure*}[t]
    \centering
    \vspace{1mm}
    \includegraphics[trim={5mm 3mm 4mm 3mm}, width=0.65\linewidth]{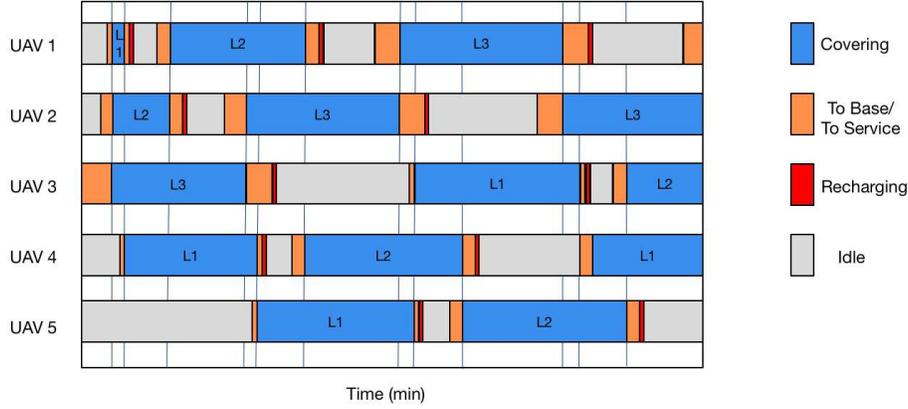}
    \vspace{-0mm}
    \caption{UAV recharge scheduling by using {\sc HeRR} with $I=3$, $f=45$~min, $c=15$~s and $\{g_i\} = \{1,5,9\}$~min.}
    \vspace{-3mm}
    \label{fig:HeRR_Ex1}
\end{figure*}

 Thus, every time a UAV in aerial location $i_k\!\in \!\mathcal{I}$ needs to be replaced, it is sufficient to have $n_k$ backup UAVs. Hence, in general, the sufficient amount of auxiliary UAVs is $\max\limits_{1\leq k\leq I} \{n_k\}$, while other $I$ UAVs are actually providing service.
 Hence, the theorem follows.
\end{proof}

We note that, under homogeneous conditions, the proof of Theorem~\ref{the:HeRR} is equivalent to the proof of Theorem~\ref{the:HoRR}. That is, Theorem~\ref{the:HeRR}, when applied to a homogeneous scenario, provides the same optimal number of UAVs as Theorem~\ref{the:HoRR}.
Furthermore, in Appendix~\ref{app:fromto} we
 show that the value provided by Theorem~\ref{the:HeRR} is very close to the actual number of drones used by {\sc HeRR}.

Regarding the complexity of finding $n_k$, in the next lemma we show that it is at most logarithmic in the input parameters.

\begin{lemma}
For all $1\!\leq \! k\! \leq \!I$,  obtaining $n_k$  has a complexity that is at most logarithmic as $\mathcal{O}\!\left(\log A_k\right)$, where
  \vspace{1mm}
  \begin{eqnarray}
    A_k = \left\lceil\frac{g_{i_k}\!+\!c\!+g_{i_{k^*}}}{f-2g_{i_I}} \sum\limits_{l=1}^I g_{i_l}\bigg/g_{i_1}\right\rceil\!\!. \nonumber
  \end{eqnarray}
  \label{o:Ak}
\end{lemma}
\begin{proof}
  In Theorem~\ref{the:HeRR} we need to find $n_k$ as the minimum natural number $n$ accomplishing the indicated inequality. Hence, if we find for all $k$ some natural number $A_k$ such that the inequality is guaranteed to hold, the search space for natural numbers gets reduced to the finite set $\{1,\hdots, A_k\}$ and the complexity of finding the minimum $n$ would be at most logarithmic with $A_k$. Hence, we find such natural number $A_k$.

  In the proof of Theorem~\ref{the:HeRR}, we need that $n_k$ verifies that:
  \begin{eqnarray}
    \sum\limits_{l=k+1}^{k+n_k} x_{i_l} = \frac{f-2g_{i_I}}{\sum\limits_{l=1}^N g_{i_l}} \sum\limits_{l=k+1}^{k+n_k} g_{i_l} \geq g_{i_k} \!+\!c\!+ g_{i_{k^*}}. \nonumber
  \end{eqnarray}

  Since $\{g_{i_j}\}$ are sorted in increasing order, $g_{i_l}\!\geq\! g_{i_1}$ for all~$l$, and hence the following inequality holds:
  \vspace{1mm}
  \begin{eqnarray}
    \frac{f\!-\!2g_{i_I}}{\sum\limits_{l=1}^I g_{i_l}} \!\!\sum\limits_{l=k+1}^{k+n_k} \!\!g_{i_l} \geq
    \frac{f\!-\!2g_{i_I}}{\sum\limits_{l=1}^I g_{i_l}} \!\!\sum\limits_{l=k+1}^{k+n_k} \!\!g_{i_1} = \frac{f\!-\!2g_{i_I}}{\sum\limits_{l=1}^I g_{i_l}} n_k g_{i_1}. \nonumber
  \end{eqnarray}
  \vspace{1mm}

  Hence, if $n_k\geq \frac{g_{i_k}+c+g_{i_{k^*}}}{f-2g_{i_I}}\cdot \sum_{l=1}^I g_{i_l} \big/ g_{i_1}$, we might not get the minimum number of needed auxiliary drones $n_k$ needed by \textsc{HeRR}~but instead get an upper bound $A_k$, for all $1\!\leq  \! k\!\leq \!I$. Thus, we define $A_k$ as:
  \vspace{1mm}
  \begin{eqnarray}
    A_k = \left\lceil \frac{g_{i_k}\!+\!c\!+g_{i_{k^*}}}{f-2g_{i_I}} \sum\limits_{l=1}^I g_{i_l}\bigg/g_{i_1}\right\rceil\!\!. \nonumber
  \end{eqnarray}

  Hence, the lemma follows.
\end{proof}

In Figure~\ref{fig:HeRR_Ex1}, we show an illustrative example of how the {\sc HeRR} routine works on a set formed by three locations so that $f=45$~min, $\{g_i\} \!=\! \{1,5,9\}$ and $c=5$~min. In that case, Theorem~\ref{the:HeRR} tells us that two additional UAVs are enough to guarantee a persistent service at these three locations (i.e., $M=5$). It can be seen that now the time instants at which UAVs go to recharge are not homogeneously spaced.



\subsection{The {\sc PHeRR}  algorithm}
\label{s:eHeRR}

A feature that characterizes how {\sc HeRR} works is that the different locations are covered by the UAVs in a rotating fashion. Furthermore, all locations are covered during the same time interval, which is given by the maximum flight time of the UAVs, minus twice the displacement time to go to the furthest location (i.e., $f - 2 g_{i_I}$). Clearly, this results in all the locations being influenced by the furthest one, which could be quite unsuitable in very heterogeneous scenarios.


Let us illustrate what we just said with a simple example. Assume a scenario where we want to cover five locations with displacement times given by $\{g_1,g_2,g_3,g_4,g_5\}=\{5,6,9,10,15\}$ (taking $f=45$~min and $c=15$~s). By directly applying Theorem~\ref{the:HeRR} to this example, we will obtain that the required number of UAVs is $14$. However, if we partition these locations into 3 sets with \emph{similar} displacement times, one formed by  locations $1$ and $2$, another formed by locations $3$ and $4$, and another formed by location $5$, and apply Theorem~\ref{the:HeRR} to each set, we will obtain that the number of required UAVs is $11$: $3$ UAVs to cover locations $1$ and $2$; $4$ UAVs to cover locations $3$ and $4$; and $4$ UAVs to cover location $5$.

Next, we formulate the combinatorial problem to obtain the best partition of the locations so that, by using the {\sc HeRR} routine on each of the obtained sets, the resulting number of UAVs is the minimum.

\begin{newproblem}
Assume a fleet of UAVs that provide service to a heterogenous scenario characterized by parameters $f$, $c$ and $g_i$ (for each $i\!\in\!\mathcal{N}$).  Find a partition $\mathcal{P}_{\mathcal{N}} = \{\mathcal{I}_1,\hdots,\mathcal{I}_N\}$
so that, by applying the {\sc HeRR} routine to each element of the partition, the resulting total number of UAVs is the minimum.
\label{def:Partition}
\end{newproblem}

Unfortunately, partition problems such as the one we presented above are known to be NP-hard~\cite{chopra1993partition}. Therefore, here we introduce a heuristic algorithm, which we call \emph{Partitioned Heterogeneous Rotating Recharge} ({\sc PHeRR}), that works in linear time.

The code of {\sc PHeRR} is shown in Algorithm~\ref{alg:PHeRR}.
It works as follows: First, it sets the initial partition as the whole set of locations and computes the amount of needed UAVs, $M$ (Steps~\ref{step:initialpartition} to~\ref{step:firstM}).
Then, at each iteration of the while loop, the algorithm takes the subset of the current partition that contains more locations and splits it into two new subsets by moving the furthest location to a separate subset (Steps~\ref{step:newP'} to~\ref{step:PP+1}). This is done because, as mentioned earlier,  the number of UAVs found by \textsc{HeRR}  is affected by the furthest location. The resulting new partition is evaluated (Steps~12 to~13) and the  process is repeated until the total number of UAVs required becomes higher than with the previous configuration. This leads to find a (local) minimum.
Finally, the {\sc HeRR} routine is applied to each one of the subsets of the partition that requires the smallest number of UAVs among the probed partitions  
(Step~\ref{step:HeRRtoeachIp}).

At this point, we would like to point out that we have also compared the linear search of partitions that we use with the solution provided by a full combinatorial search (which is not feasible in practice, since it takes a lot of time). As we show in Appendix~\ref{app:fromto}, the difference among them is almost negligible.

Note that, in case of addressing a homogeneous scenario, the \textsc{PHeRR} algorithm will provide the same schedule as \textsc{HoRR} and hence, it will provide optimal results. Indeed, since in that case all the displacement times are the same, then the initial partition of \textsc{PHeRR} contains all locations with equal displacement times, and no other partition will be checked (indeed, no other partition could provide a lower total number of UAVs). In such a homogeneous case, as noted before, Theorem~\ref{the:HeRR} finds the optimal number of UAVs.

\begin{algorithm}[t]
  \caption{\small \emph{Partitioned Heterogeneous Rotating Recharge} ({\sc PHeRR})\normalsize}
  \label{alg:PHeRR}
  \begin{algorithmic}[1]
    \REQUIRE{${\cal N}$, $f$, $\{g_i\}_{i\in\mathcal{N}}$, $c$.}
    \STATE Set the initial partition $\mathcal{P}=\mathcal{N}$, with its elements increasingly ordered in accordance with $\{g_i\}_{i\in\mathcal{N}}$.  \label{step:initialpartition}
    \STATE Set the partition size $P =  1$.
    \STATE Define the only set of the partition $\mathcal{P}$ as $ \mathcal{I}_1$.
    \STATE Apply Theorem~\ref{the:HeRR} to $\mathcal{P}$ and set $M$ and $M_{new}$ to the provided value. \label{step:firstM}
    \WHILE{$M_{new} \leq M$} \label{step:whiledo}
    \STATE Set $M \leftarrow M_{new}$.
    \STATE Derive a new partition $\mathcal{P}'$ of subsets $\mathcal{I}^{'}_p$, $\forall \; 1 \!\leq\! p\!\leq\! P\!+\!1$: \label{step:newP'}
    \STATE \quad Set $\mathcal{I}^{'}_p \leftarrow \mathcal{I}_{p-1}$, $\forall \; 3 \!\leq\! p\!\leq\! P\!+\!1$.
    \STATE \quad $\mathcal{I}^{'}_2 \leftarrow \max \{ \mathcal{I}_1 \}$
    \STATE \quad $\mathcal{I}^{'}_1 \leftarrow \mathcal{I}_1 - \mathcal{I}^{'}_2$.
    \STATE \quad Set $\mathcal{P} \leftarrow \mathcal{P}^{'}$ and $P \!\leftarrow\! P\!+\!1$. \label{step:PP+1}
    \STATE Obtain the number of UAVs $M_p$ used for each subset $\mathcal{I}_p\!\in\!\mathcal{P}$, $\forall \; 1\!\leq\! p\!\leq\! P$ (by applying Theorem~\ref{the:HeRR} to each $\mathcal{I}_p$)
    \STATE Derive $M_{new} \leftarrow \sum\limits_{p=1}^P M_p$.
    \ENDWHILE \label{step:endwhiledo}
    \STATE Apply {\sc HeRR} to each $\mathcal{I}_p\in\mathcal{P}$. \label{step:HeRRtoeachIp}
  \end{algorithmic}
\end{algorithm}
\setlength{\textfloatsep}{15.5pt}

\section{Numerical analysis of the \textsc{PHeRR} algorithm}
\label{s:perf}

As we have previously done in the case of the {\sc HoRR} algorithm, in this section we numerically analyze the performance of the \textsc{PHeRR} algorithm.



%
%
%

\subsection{Effect of heterogeneity}
\label{ss:hetero}

What distinguishes a homogeneous scenario from a heterogenous one is the fact that, in the latter case, the displacement times from the \emph{RS} to the locations can be different. Therefore and in order to characterize the \emph{heterogeneity} of a given scenario, we define its \emph{displacement deviation} (denoted by $\Delta$) as the maximum displacement time deviation of any location over the average displacement time $\overline{g}$.  Hence, when $\Delta\!=\!0$ we are in a homogeneous scenario, and the higher the $\Delta$ value, the higher the heterogeneity.

In Figure~\ref{fig:PHeRR_heterogeneity}, we parallelize the analysis performed in Figure~\ref{fig:HoRR_beh} in the homogeneous case but for the heterogeneous case.
More specifically, we fix an average displacement time $\overline{g}$ and draw $g_i$ values according to a uniform random variable $\mathcal{U}\left(\overline{g}\left(1\!-\!\Delta\right)\!,\, \overline{g}\left(1\!+\!\Delta\right)\right)$.
Then, for each value of $f$, we vary the value of $\Delta$ from $0$ to~$0.5$, which results in a band of lines of degrading color tone in the figure, the lower envelop of the band being the performance in the homogeneous case. For each value of $\Delta$, we used \textsc{Matlab} to simulate $1.000$ different realizations of the random heterogeneous scenario and computed average results.

Beside using $c=15$~s, in our simulations we assumed that $\overline{g}=5$~min. By using a UAV with a speed of $72$~km/h---which is fairly conservative---this corresponds to a distance of $6$~km from the \emph{RS}. That will allow us to consider scenarios with a wide range of heterogeneities. For instance, by using a displacement deviation value of $\Delta=0.3$, the UAVs can be placed at distances between $4.2$ and $7.8$~km from the \emph{RS}, and by using $\Delta=0.5$ the UAVs can be placed at distances between $3$ and $9$~km. Nevertheless, we have also performed simulations with different values of both  $c$ and $\overline{g}$, and we observed that the shapes were similar.



\begin{figure}
\vspace{-8mm}
\includegraphics[width=9cm]{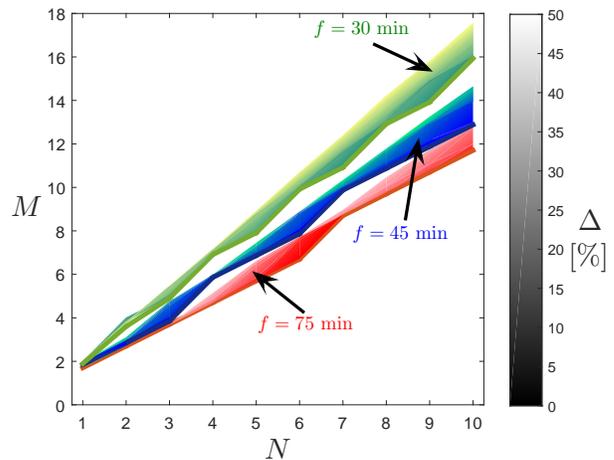}
    \caption{Behaviour of the \textsc{PHeRR} algorithm. $\overline{g}=5$~min and $c=15$~s. Different hues show the results for different values of $\Delta$, from $\%0$ to $50\%$.}
\label{fig:PHeRR_heterogeneity}
\vspace{-2mm}
\end{figure}

First of all, in Figure~\ref{fig:PHeRR_heterogeneity}, it can be readily seen that the more we increase the value of $\Delta$, the more the value of $M$ increases (for the same number of locations, $N$).
This behaviour matches with the fact that,
as it has been already shown in Theorem~\ref{the:LBhomo}, covering in heterogeneous scenarios is, in general, more costly than covering in homogeneous ones (i.e., $M_{het} \!\geq\! M_{hom}$).
However, it can also be observed that the increase in the value of $M$ with $\Delta$  is quite moderate. Indeed, in most cases, only one additional UAV (with respect to the homogeneous case) was required, and even in stressful conditions (namely, with $f=30$~min, $\Delta \approx 0.5$ and more than $8$~locations), only two additional UAVs were enough.

Whereas our analysis shows that heterogeneity is not a factor that significantly affects system performance in most cases, it must be taken into account the possibility of finding a scenario that greatly increases the number of UAVs. Anyhow, with the realistically vast range of scenarios represented in Figure~\ref{fig:PHeRR_heterogeneity}, we have found that the average values of the UAV fleet size increase only, on average, by one or two units with respect to the homogeneous case.


\subsection{Effect of overhead}

Next, we evaluate the performance of {\sc PHeRR} regarding the number of UAVs necessary to guarantee that a given set of locations are covered in a persistent manner (i.e., $M$) against the lower bound provided by Theorem~\ref{th:LBvin} (i.e., against $M_{LB}$).
For such a task, we define the \emph{approximation factor} of {\sc PHeRR} against the lower bound
as the ratio between $M$ and $M_{LB}$. Clearly, the closer the value of the approximation factor to~$1$, the better the result.

Recall that we assume that the maximum flight time of the UAVs is $f$. We define the \emph{relative overhead} of location~$i \!\in\! {\cal N}$ as $\omega_i\!=\!\frac{2\,g_i}{f}$.
Roughly speaking, $\omega_i$ indicates the fraction of time that a UAV will use to fly from the \emph{RS} to location $i$ and come back. We also define the \emph{average relative overhead}, or just \emph{overhead}, as \blue{$\omega = Avg(\omega_i) = \frac{2\,\overline{g}}{f}$}. Then, by fixing the flight time and varying the displacement times, we can model scenarios with different overheads.

In Figure~\ref{fig:DeltaOmegaN10HeL}, we show the approximation factor of \textsc{PHeRR}. We show only average results because the observed variability is very low and cannot be well appreciated in the figure.
We fix $f=45$~min and $c=15$~s and, for each value of $\omega$, derive the corresponding value of $\overline{g}$, on top of which we apply a deviation $\Delta$ between $0$ and $0.5$,  as explained before.
We have also considered two different fleet sizes: $N\!=\!10$ and $N\!=\!15$. Before we proceed with the analysis of the results, it must be taken into account that the values $M_{LB}$ provided by Theorem~\ref{th:LBvin} are not guaranteed to be optimal, \blue{and the real optimal could be greater than $M_{LB}$}. So, the values obtained for the \emph{approximation factor} are pessimistic, in the sense that they represent upper bounds (i.e., real values could be  smaller).

\begin{figure}
    \centering
    \includegraphics[width=9cm]{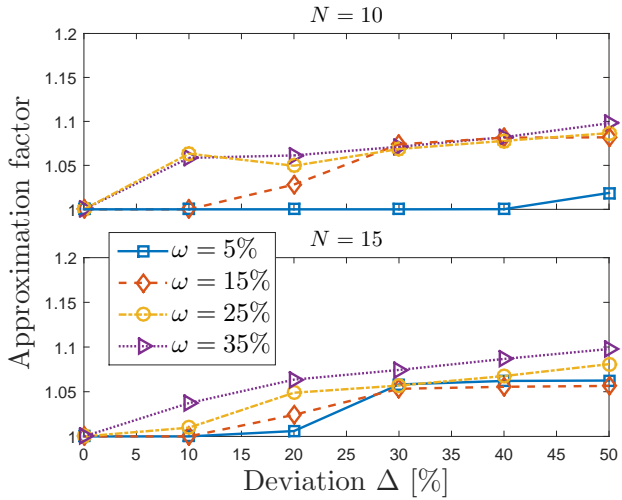}
    \caption{Impact of displacement time deviation $\Delta$ on the approximation factor of \textsc{PHeRR} w.r.t. \textsc{LB}. $f=45$~min, $c=15$~s.}
    \label{fig:DeltaOmegaN10HeL}
    \vspace{-2mm}
\end{figure}

It can be seen that the approximation factor increases with the heterogeneity. However, such an increase occurs in a smooth way and quickly stabilizes. This behavior is compatible with our results in Section~\ref{ss:hetero}. This confirms that heterogeneity is not a factor that significantly affects system performance.

Furthermore, Figure~\ref{fig:DeltaOmegaN10HeL} also shows that {\sc PHeRR}  provides very good results, with  approximation factors always below $1.1$.
This is much better than previous results~\cite{8901097,9213932}, which respectively achieved, on average, approximation factors of $1.5$ or $1.7$.
More precisely, it can be observed that approximation factors close to $1.1$ occur in stressful conditions, with large values of both $\Delta$ and $\omega$.
That is, approximation factors close to $1.1$ only occur in very heterogeneous scenarios in which the UAVs must use a significant amount of energy to fly
to/from the locations.
In contrast, when the conditions are less stringent, {\sc PHeRR}  provides near-optimal results, to the point where it is optimal in homogeneous scenarios or scenarios with very small overhead.

\subsection{Tightness of $M_{LB}$}

In Theorem~\ref{th:LBvin} we obtained a lower bound on the minimum number of UAVs necessary to guarantee that $N$ of them will be always providing service. Then, in the previous subsection we have used it to analyze the performance of the {\sc PHeRR} algorithm. However, since the values provided by {\sc PHeRR} are very close to $M_{LB}$ (see Figure~\ref{fig:DeltaOmegaN10HeL}), this collaterally implies that the lower bound provided by Theorem~\ref{th:LBvin}  is very close to be optimal.

\section{Conclusions}
\label{s:conclusions}

In this article, we have studied the problem of the UAV fleet recharge scheduling, \blue{meant to minimize the fleet size while providing persistent service in a set of aerial locations}. We considered two scenarios: On one hand, we designed a simple scheduling mechanism for UAVs serving aerial locations with homogeneous distances to a recharge station, and we proved that it is feasible and optimal. On the other hand, we demonstrated that the problem becomes NP-hard when the aerial locations are non-evenly distributed. Then, we \blue{derived a very tight lower bound for the UAV fleet size and} designed a lightweight recharging scheduling scheme, which was shown to be \blue{not only much better than state-of-the-art heuristics but also near-optimal}.

\section*{Acknowledgment}

This work has been partially supported by the Region of Madrid through the TAPIR-CM project (S2018/TCS-4496).

\bibliographystyle{IEEEtran}
\bibliography{biblio}
\vspace{-14mm}
\begin{IEEEbiography}[{\vspace{0mm}\includegraphics[width=0.8in,height=1.2in]{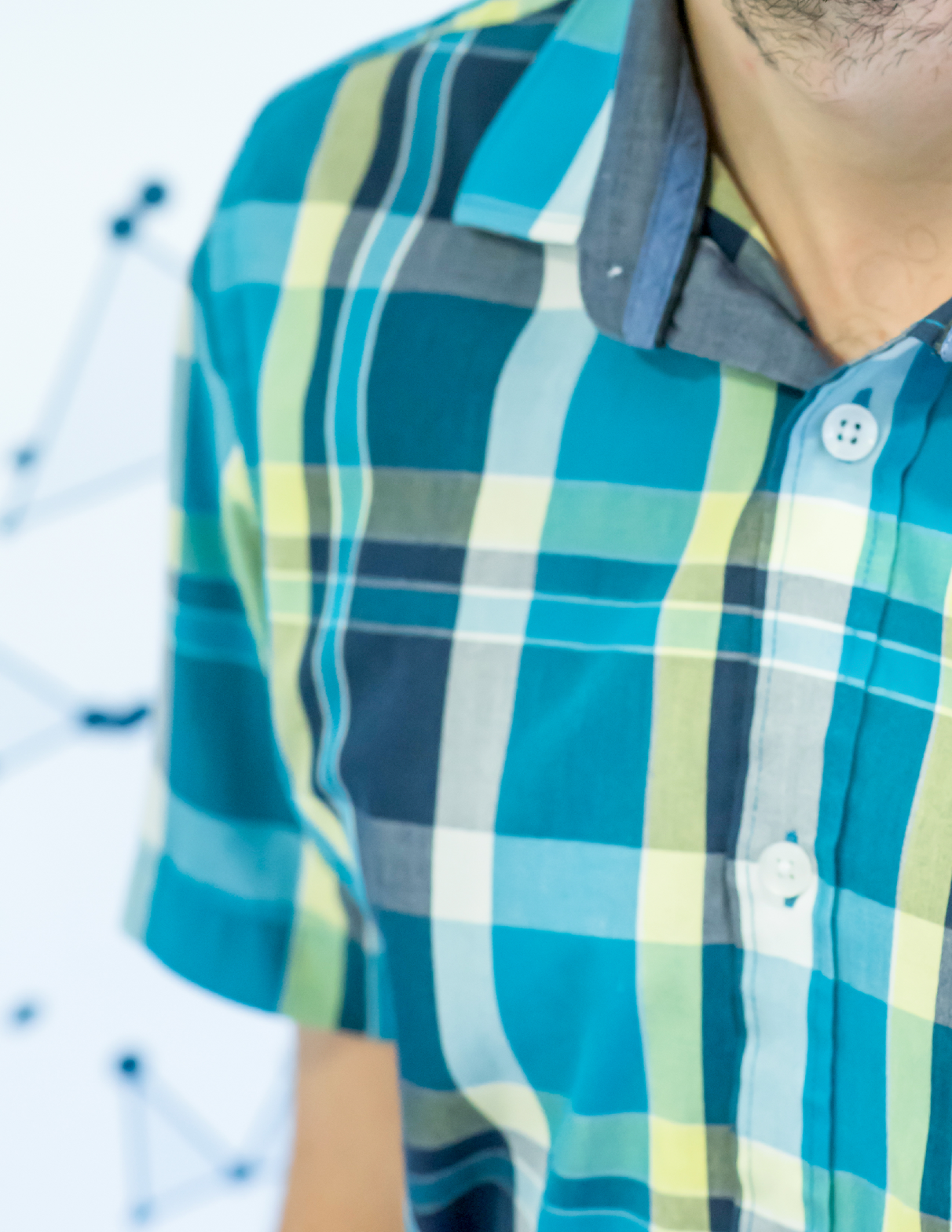}}]{Edgar~{Arribas}}
graduated in Mathematics from the Universitat de Val\`encia and received his PhD in Telematic Engineering in 2020 at IMDEA Networks Institute and Universidad Carlos III de Madrid, funded by the MECD FPU grant. He is currently a lecturer and researcher at the Applied Mathematics and Statistics Department of Universidad CEU San Pablo (Spain).
He works on optimization of dynamic relay in wireless networks.
\end{IEEEbiography}

\vspace{-14mm}
\begin{IEEEbiography}[{\vspace{-11mm}\includegraphics[width=0.95in,height=1.25in,keepaspectratio]{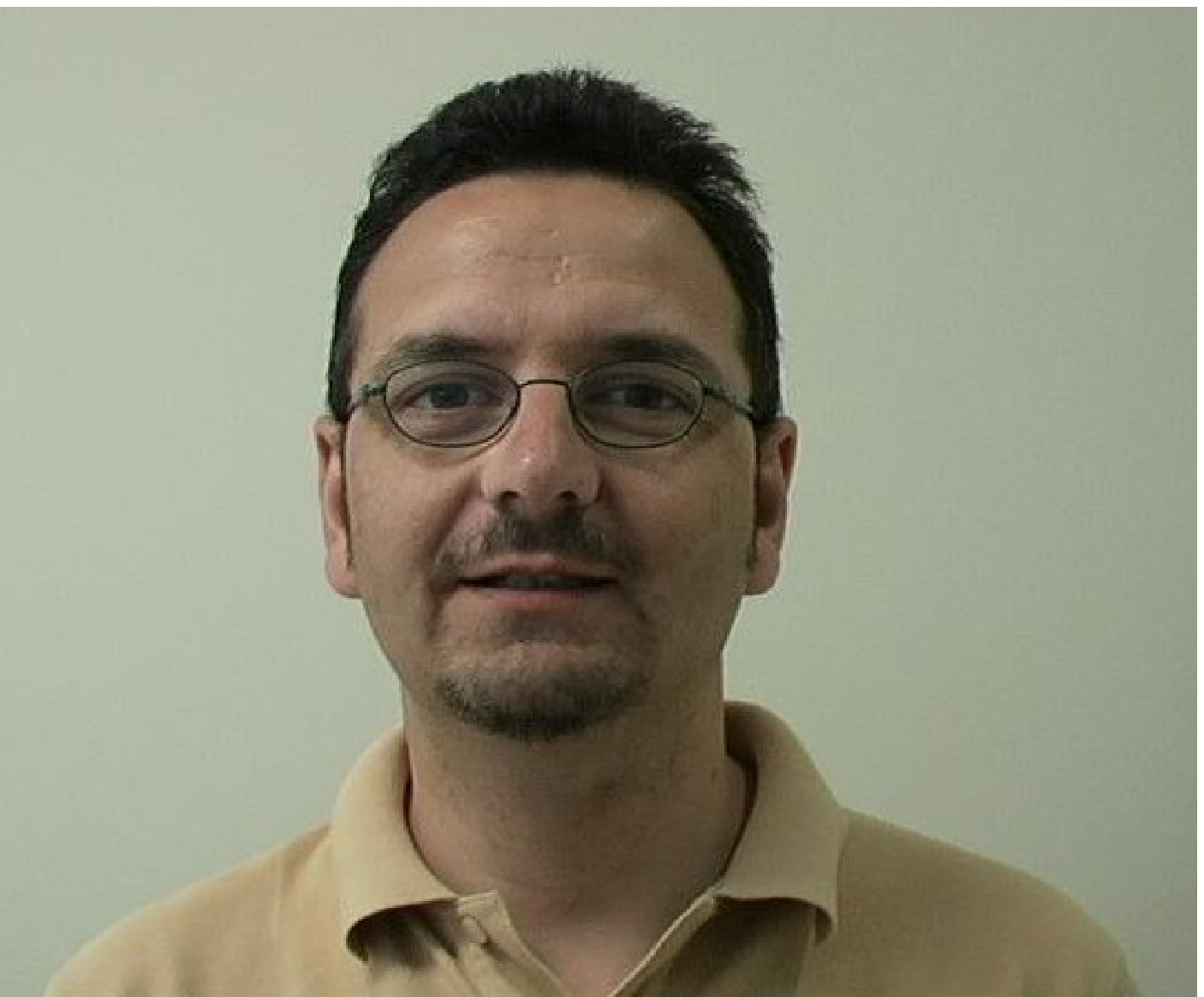}}]{Vicent~{Cholvi}}
graduated in Physics from the University of Valencia, Spain and received his doctorate in Computer Science in 1994 from the Polytechnic University of Valencia. In 1995, he joined the Jaume~I University in Castell\'on, Spain where he is currently a Professor. His interests are in distributed and communication systems.
\end{IEEEbiography}

\vspace{-14mm}
\begin{IEEEbiography}[{\includegraphics[width=1in,height=1.1in,keepaspectratio]{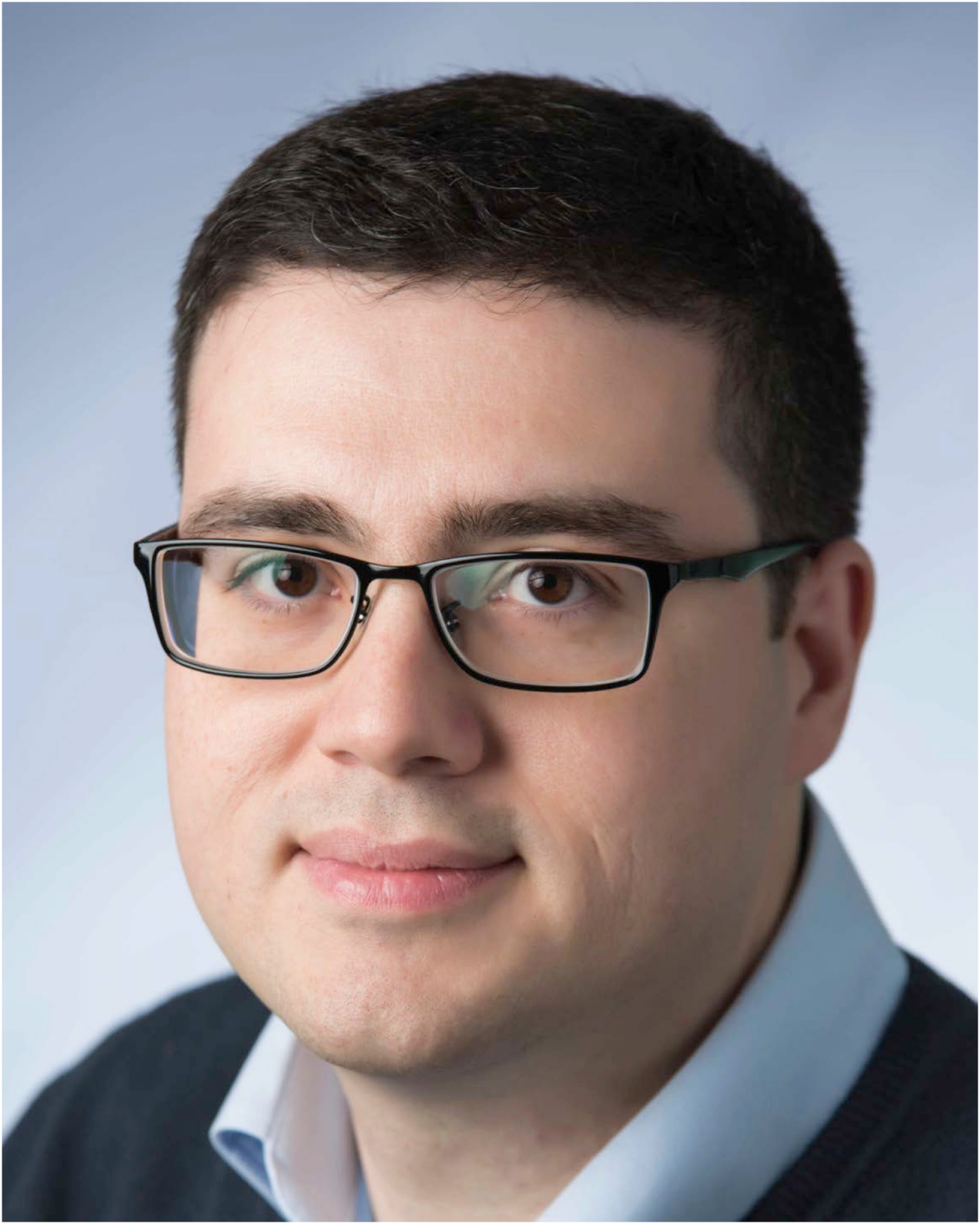}}]{Vincenzo~{Mancuso}}
is Research Associate Professor at IMDEA Networks, Madrid, Spain, and recipient of a Ramon y Cajal research grant
of the Spanish Ministry of Science and Innovation.
Previously, he was with INRIA (France), Rice University (USA) and University of Palermo (Italy), from where he obtained his Ph.D. in 2005.
His research focus is on analysis, design, and experimental evaluation of
opportunistic wireless architectures and mobile broadband services.
\end{IEEEbiography}






\appendices
\section{NP-Hardness of the \emph{UPS} problem}
\label{App:NP-completeness}

The authors in~\cite{9213932,8901097} demonstrate that the \emph{ Bin Maximum Item Double Packing} (\emph{BMIDP}) problem exactly solves what they define as the \emph{Minimal Spare
Drones for Persistent Monitoring} (\emph{MSDPM}) problem with one recharging station. The \emph{MSDPM} problem with one recharging station is equivalent to the \emph{UPS} problem when the recharging time $c$ is zero and $2g_i < f/2$. Here, we first formulate the \emph{BMIDP} problem defined in~\cite{8901097}, and then we prove that it is NP-hard.

\begin{definition}[\emph{Bin Maximum Item Doubled Packing (BMIDP) problem}~\cite{9213932}]
  Given a set of items $\mathcal{I}_t=\{1,\hdots,n\}$, where each item $i\in\mathcal{I}_t$ has size $w_j\in]0, 1]$, check whether it is possible to split the items in $N\!\in\!\mathds{N}$ disjoint bins $W_1,\hdots, W_N$ of capacity $1$ where the maximum item of each bin $W_k$ must be packed twice (i.e., $\forall \; 1\!\leq\! k\!\leq\! N$, $\sum\limits_{w_j\in W_k} w_j + \max\limits_{w_j\in W_k} w_j \leq  1$).
\end{definition}

\begin{theorem}
  The BMIDP problem is NP-hard.
  \label{th:nphard}
\end{theorem}
\begin{proof}

In the following three steps, we reduce the k-way number partitioning problem (\emph{kPP})~\cite{korf2014optimal} to the \emph{BMIDP} problem.
Therefore, since it is well-known that the \emph{kPP} problem is NP-hard~\cite{garey1979computers}, so it is the \emph{BMIDP} problem.

\begin{enumerate}

\item \textsl{Reduction of kPP to BMIDP:} Given $n\!\in\!\mathds{N}$, we consider a general instance of \emph{kPP} $\mathcal{I}=\{i_1,\hdots,i_n\}$ such that $i_j\!\geq\!0$, $\forall 1\!\leq\! j\!\leq\! n$ and at least one non-zero element $i_j$. Let $I=\sum\limits_{j=1}^n i_j > 0$. Let $N\!\geq\! 1$, and let $\left\{S_k=\{i_{k_1},\hdots,i_{k_m}\}\right\}_{k=1}^N$ be a partition of the general instance $\mathcal{I}$. We now define:
\begin{align}
 \!\!\!\!\!i^{(k)} &\!=\! \max_{i_j\in S_k} i_j, \quad \forall 1\!\leq\! k\!\leq\! N; \nonumber \\
 \!\!\!\!\!w_j &\!=\! \frac{N}{I\!+\!Ni^{(k)}}i_j, \text{ if } i_j\!\in\! S_k, \quad \forall 1\!\leq\! j\!\leq\! n; \nonumber
 \label{eq:wj}\\
 \!\!\!\!\!W_k  &\!=\! \left\{w_j : i_j\in S_k\right\}, \quad \forall 1\!\leq\!  k\!\leq\!  N; \\ \nonumber
 \!\!\!\!\!w^{(k)} &\!=\!  \max_{w_j\in W_k} \!w_j = \frac{N}{I\!+\!Ni^{(k)}} i^{(k)} < 1, \; \forall 1\!\leq\! k\!\leq\! N. \nonumber
\end{align}
The last equation above also implies that
\begin{align}
i^{(k)} & = \dfrac{I}{N} \dfrac{w^{(k)}}{1-w^{(k)}}.
\label{eq:ik}
\end{align}

Note that since $\{S_k\}_{k=1}^N$ is a partition of $\mathcal{I}$, the definitions above are well defined. Note that $0\!\leq\! w_j \!<\!  1$, $\forall 1\!\leq\! j\!\leq\! n$ and that for all $j$, then $w_j\!\in\! W_k$ for some $k$ if and only if $i_j\!\in\! S_k$ for the same $k$.

We now prove that the partition $\{S_k\}_{k=1}^N$ is a solution of \emph{kPP} with $N$ partitions if and only if $\{W_k\}_{k=1}^N$ is a solution of \emph{BMIDP} with $N$ bins.

\item \textsl{Necessary condition:} For the right direction, we assume that $\{S_k\}_{k=1}^N$ is a solution of \emph{kPP}. Hence:
\begin{eqnarray}
  \sum\limits_{i_j \in S_k} i_j = \frac{I}{N}, \quad \forall 1\!\leq\! k\!\leq\! N. \nonumber
\end{eqnarray}

Now, we verify that $\{W_k\}_{k=1}^N$ is a solution of \emph{BMIDP} with $N$ bins by evaluating if the sum of all elements in a set $W_k$ plus the maximum in $W_k$, which is $w^{(k)}$, fits in a bin of size 1:
\begin{eqnarray}
  w^{(k)} +\!\! \sum\limits_{w_j \in W_k} w_j = \frac{N}{I + Ni^{(k)}}\left(\! i^{(k)} +\!\! \sum\limits_{i_j\in S_k} i_j\!\right) = \nonumber\\
  \frac{N}{I + Ni^{(k)}}\left(i^{(k)} + \frac{I}{N}\right) = 1\leq  1, \nonumber
\end{eqnarray}
which is true $\forall 1\!\leq\! k\!\leq\! N$.

\item \textsl{Sufficient condition:} For the left direction, we assume that $\{W_k\}_{k=1}^N$ is a solution of \emph{BMIDP} with $N$ bins:
\begin{eqnarray}
  w^{(k)} + \sum\limits_{w_j\in W_k} w_j \leq 1, \quad \forall 1\!\leq\! k\!\leq\! N.
  \label{eq:leftcondition}
\end{eqnarray}
Given $1\!\leq\! j\!\leq\! n$, there exists a unique $1\!\leq\! k\!\leq\! N$ such that $w_j\!\in\! W_k$.
Since $w_j\!\in\! W_k$, then $i_j\!\in\! S_k$ and moreover \eqref{eq:wj} establishes a relation between $w_j$ and $i_j$,
which can be rewritten as follows:
\begin{eqnarray}
  i_j = \frac{I+Ni^{(k)}}{N} w_j.
  \label{eq:ijwj}
\end{eqnarray}

Now, by plugging \eqref{eq:ik} into \eqref{eq:ijwj}, we obtain that
\begin{align}
i_j=\frac{I}{N}\frac{w_j}{1-w^{(k)}}, \; \forall 1\!\leq\! j\!\leq\! n, \text{ with } k \mid i_j \in S_k. \nonumber
\end{align}


Hence, for all $1\!\leq\! k\!\leq\! N$, it is satisfied that:
\begin{align}
  \sum\limits_{i_j\in S_k} i_j & = \sum\limits_{w_j \in W_k} \frac{I}{N} \frac{w_j}{1-w^{(k)}}
  =  \frac{I}{N}  \frac{\sum\limits_{w_j \in W_k} w_j}{1-w^{(k)}}
  \nonumber \\
   & = \frac{I}{N} \left(\frac{w^{(k)}+\sum\limits_{w_j\in W_k} w_j}{1-w^{(k)}} - \frac{w^{(k)}}{1-w^{(k)}}\right) \nonumber \\
   & \leq \frac{I}{N}\left(\frac{1}{1-w^{(k)}} - \frac{w^{(k)}}{1-w^{(k)}}\right) = \frac{I}{N},
   \label{eq:left_proof}
\end{align}
where we have used inequality \eqref{eq:leftcondition} in the passage from the second to the third row.

Since $I$ is defined as $\sum\limits_{j=1}^n i_j$ and $\{S_k\}_{k=1}^N$ is a partition of~$\mathcal{I}$,
\eqref{eq:left_proof} must hold as equality, i.e.:
\begin{eqnarray}
  \sum\limits_{i_j\in S_k} i_j = \frac{I}{N}, \quad \forall 1\!\leq\! k\!\leq\! N. \nonumber
\end{eqnarray}
If that were not the case, i.e., $\sum\limits_{i_j\in S_k} i_j < \frac{I}{N}$ for some $1\!\leq\! k\!\leq\! N$, then:
\begin{eqnarray}
  I = \sum\limits_{j=1}^n i_j = \sum\limits_{k=1}^N \sum\limits_{i_j\in S_k} i_j < \sum\limits_{k=1}^N \frac{I}{N} = I, \nonumber
\end{eqnarray}
which is a contradiction.
Therefore, the partition of $\mathcal{I}$, $\{S_k\}_{k=1}^N$, is a solution of the \emph{kPP} problem.
\end{enumerate}

As a result, we have found a reduction of \emph{kPP} that admits a solution with an $N$-partition if and only if \emph{BMIDP} admits solution with $N$ bins.
Hence, the theorem follows.

\end{proof}

\section{Auxiliary Results}
\label{app:MathInequality}


\begin{lemma}
  Given $N\!\in\!\mathds{N}$, and given a vector $x\!=\!(x_i)_{i=1}^N\!\in\!\mathds{R}^N$ such that $x_i > 0$, $\forall i=1,\hdots,N$, then:
  \begin{eqnarray}
    \sum\limits_{i=1}^N x_i \cdot \sum\limits_{i=1}^N\frac{1}{x_i} \geq N^2. \nonumber
  \end{eqnarray}
  \label{l:sum2}
\end{lemma}
\begin{proof}
  First, we do some algebraic manipulation:
  \begin{eqnarray}
    \sum\limits_{i=1}^N x_i \cdot  \sum\limits_{i=1}^N \frac{1}{x_i} = \sum\limits_{i=1}^N \sum\limits_{j=1}^N \frac{x_i}{x_j} = N + \sum\limits_{i\neq j} \frac{x_i}{x_j}. \nonumber
    \label{eq:sum1}
  \end{eqnarray}

  Now, let $S=\{s: s = \frac{x_i}{x_j} > 1 \text{ for some } i\neq j\}$.
  The cardinality of $S$ is the number of pairs $(x_i, x_j)$ with $x_i > x_j$, that is:
  \vspace{-1mm}
  \begin{eqnarray}
    |S| = \frac{N(N-1)}{2}.
  \end{eqnarray}

  Now, we take again Eq.~\eqref{eq:sum1} and express $\sum\limits_{i\neq j}\frac{x_i}{x_j}$ in the terms of set $S$ by considering
  that for each pair of values $(x_i, x_j)$ that have ratio $s>1$, we also have the pair  $(x_j, x_i)$ with ratio $1/s < 1$, while the ratio is 1
  in the $N$ different cases in which $i=j$:
  \vspace{-1mm}
  \begin{eqnarray}
    \sum\limits_{i=1}^N x_i \cdot  \sum\limits_{i=1}^N \frac{1}{x_i} = N + \sum\limits_{s\in S} \left( s + \frac{1}{s} \right). \nonumber
  \end{eqnarray}
 Since the function $s +  1/s$ of positive argument $s$ has a derivative that becomes zero at $s=1$, where the function assumes value 2,
 and its second derivative is always positive, we can conclude that the function has a minimum whose value is 2, so that
 $s +  1/s \ge 2, \forall s \in \mathbb{R}^+$. Therefore, we have:
  \begin{eqnarray}
    \sum\limits_{i=1}^N x_i \cdot  \sum\limits_{i=1}^N \frac{1}{x_i} \geq N + \sum\limits_{s\in S} 2 = N + 2|S| = N^2. \nonumber
  \end{eqnarray}

  Hence, the lemma follows.
\end{proof}

\begin{theorem}
\label{th:sumNsumsum}
Given $N\!\in\!\mathds{N}$, and given two vectors $x=(x_i)_{i=1}^N$, $y=(y_i)_{i=1}^N\in\mathds{R}^N$ such that $x_i, y_i > 0$, $\forall i=1,\hdots,N$, then:
\begin{eqnarray}
  \sum\limits_{i=1}^N \frac{x_i}{y_i} \geq N\cdot \frac{\sum\limits_{i=1}^N x_i}{\sum\limits_{i=1}^N y_i}.
  \label{eq:th_res}
\end{eqnarray}
\end{theorem}
\begin{proof}
  Let $Avg(\cdot )$ be the arithmetic mean function, which can be seen as the stochastic average for a vector of equiprobable values, i.e., given a vector $z=(z_i)_{i=1}^N$, $Avg(z) = \frac{1}{N}\sum\limits_{i=1}^N z_i$. Hence,
using the conditional average formula on the expression for the vector  $x/y = (x_i/y_i)_{i=1}^N$, we have:
  \begin{align}
    & Avg\left(\dfrac{x}{y}\right)
    =  \sum\limits_{i=1}^N \dfrac{1}{N} Avg\left(\left.\dfrac{x}{y} \right |  y=y_i\right)
    = \sum\limits_{i=1}^N\frac{1}{N}Avg\left(\frac{x}{y_i}\right)  \nonumber \\
    & = Avg(x) \cdot \sum\limits_{i=1}^N \frac{1}{N}\cdot \frac{1}{y_i}
    = Avg(x)\cdot Avg\left(\frac{1}{y}\right),
    \label{eq:sum3}
  \end{align}
  where $\frac{1}{y}$ is a vector $(1/y_i)_{i=1}^N$ of positive numbers.

  Thereby, according to Lemma~\ref{l:sum2}, $\sum\limits_{i=1}^N y_i \cdot \sum\limits_{i=1}^N \frac{1}{y_i} \geq N^2$. Hence, $\frac{1}{N}\sum\limits_{i=1}^N y_i \cdot  \frac{1}{N}\sum\limits_{i=1}^N \frac{1}{y_i} \geq 1$. This means that $Avg(y) \cdot  Avg\left(\frac{1}{y}\right) \!\geq\! 1$. Hence:
  \begin{eqnarray}
    Avg\left(\frac{1}{y}\right) \geq \frac{1}{Avg(y)}.
    \label{eq:avgs}
  \end{eqnarray}

  Hence, by applying \eqref{eq:avgs} to \eqref{eq:sum3}, we can lower-bound $\frac{1}{N}\sum\limits_{i=1}^N \frac{x_i}{y_i}$ as follows:
  \begin{eqnarray}
    \frac{1}{N}\sum\limits_{i=1}^N \frac{x_i}{y_i}
    = {Avg(x)} \cdot  {Avg\left(\dfrac{1}{y}\right)}
    \geq \dfrac {Avg(x)} {Avg(y)}
    =
    \frac{\frac{1}{N}\sum\limits_{i=1}^N x_i}{\frac{1}{N}\sum\limits_{i=1}^N y_i}
    \label{eq:sum4}
  \end{eqnarray}

From Eq.~\eqref{eq:sum4} we finally get \eqref{eq:th_res} by multiplying both sides of the inequality by $N$, and then the theorem follows.
\end{proof}

\section{From \textsc{HeRR} to \textsc{PHeRR}}
\label{app:fromto}


In this section, we use the variables $\Delta$ and $\omega$ defined in Section~\ref{s:perf}. In Figure~\ref{fig:SuffvsHeRRvsPHeRRvsapHeRR}, we compare the performance of the proposed solution schedule for the \emph{UPS} problem.

Firstly, we show that the sufficient number of drones that we have deterministically derived in Theorem~\ref{the:HeRR} so that the \textsc{HeRR} routine is feasible (denoted as the \textsc{Suff} schedule) is very accurate to the actual number of drones required by the \textsc{HeRR} operation. In particular, we see that for different heterogeneity settings (for $\Delta=0.3$, $0.5$) and for diverse average overhead $\omega$, the average difference between \textsc{Suff} and \textsc{HeRR} is always negligible (below $1$\%). Hence, we find that in order to estimate in advance the number of drones required to run any \textsc{HeRR}-partition based scheduling, it is advisable to check the number of drones required by \textsc{Suff} (using Theorem~\ref{the:HeRR}).



Secondly,
in this figure
we
also
compare the \textsc{PHeRR} schedule with another \textsc{HeRR}-partition based schedule: the \textsc{Optimally-Partitioned HeRR}~ (\textsc{OPHeRR}) schedule.  With \textsc{OPHeRR}, we optimally solve \emph{the heterogeneous partition problem} defined in Section~\ref{def:Partition} by means of listing all possible partitions (i.e., a combinatorial number of options) and hence provide the best performance a \textsc{HeRR}-partition based scheduling can have.
Hence, we show the average performance comparison between \textsc{PHeRR} and \textsc{OPHeRR}, in order to show the general behaviour of the proposed schedulings. Here, we observe that also on average, there is very small difference between a linear search of partitions from \textsc{PHeRR} and the solution provided by \textsc{OPHeRR} with a full combinatorial search. Hence, the performance of \textsc{PHeRR} could be barely improved by means of any \textsc{HeRR}-partition based scheduling, which remarks the accurateness achieved with the very lightweight and linear search of partitions performed by \textsc{PHeRR}.

Finally,
the figure
shows significant average differences between the \textsc{HeRR} and \textsc{PHeRR} schedules performance, which highlights the fact that the very lightweight extra complexity added to \textsc{PHeRR} is worth it. Specially, in cases with high overhead and high heterogeneity, the difference between both schemes is not only remarkable, but we also observe that the \textsc{HeRR} results are more spread (see the standard deviation identified with error bars) than the \textsc{PHeRR} results (with smaller standard deviations). Hence, there are many instances of the problem in which the difference between the \textsc{HeRR} and \textsc{PHeRR} performance is even higher than the observed average difference.

Therefore, we conclude that the \textsc{PHeRR} schedule stands as the best option to be adopted in order to find near-optimal solutions to the \emph{UPS} problem. 

\begin{figure}
    \centering
    \includegraphics[width=9cm]{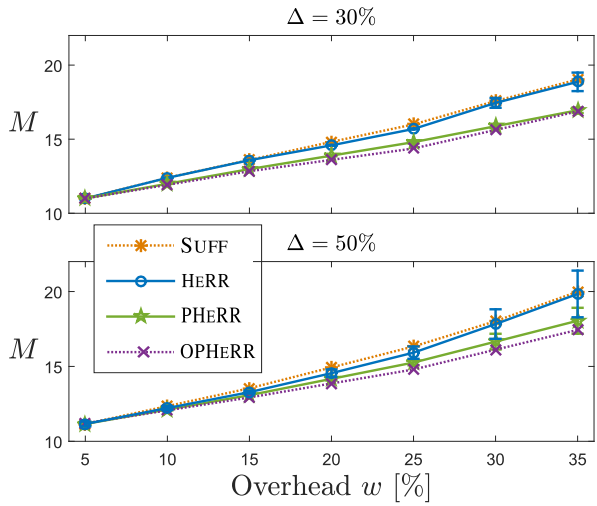}
    \caption{Performance of the \textsc{Suff}, \textsc{HeRR}, \textsc{PHeRR} and \textsc{OPHeRR} schedules. $N=10$.}
    \label{fig:SuffvsHeRRvsPHeRRvsapHeRR}
\end{figure}



\end{document}